\documentclass[11pt]{article}
\usepackage[dvips]{graphicx}
\usepackage{amssymb,amsmath,color}
\usepackage{url}

\usepackage[mathcal]{eucal}

\usepackage[colorlinks=true, allcolors=blue, pagebackref]{hyperref}       % hyperlinks
\renewcommand*{\backrefalt}[4]{%
    \ifcase #1 \footnotesize{(not cited)}%
    \or        \footnotesize{(cited on page~#2)}%
    \else      \footnotesize{(cited on pages~#2)}%
    \fi}

\DeclareMathAlphabet\mathbfcal{OMS}{cmsy}{b}{n}

\newcommand{\BEAS}{\begin{eqnarray*}}
\newcommand{\EEAS}{\end{eqnarray*}}
\newcommand{\BEA}{\begin{eqnarray}}
\newcommand{\EEA}{\end{eqnarray}}
\newcommand{\BEQ}{\begin{equation}}
\newcommand{\EEQ}{\end{equation}}
\newcommand{\BIT}{\begin{itemize}}
\newcommand{\EIT}{\end{itemize}}
\newcommand{\BNUM}{\begin{enumerate}}
\newcommand{\ENUM}{\end{enumerate}}
\newcommand{\BA}{\begin{array}}
\newcommand{\EA}{\end{array}}

\newcommand{\sign}{\mathop{ \rm sign}}
\newcommand{\idm}{I}
\newcommand{\rb}{\mathbb{{R}}}

\newcommand{\ds}{\displaystyle }
\newcommand{\BlackBox}{\rule{1.5ex}{1.5ex}}  % end of proof

\newcommand{\bs}{\overline{c} }
\newcommand{\ws}{\underline{c} }

\newenvironment{proof}{\par\noindent{\bf Proof\ }}{\hfill\BlackBox\\[2mm]}

\newtheorem{proposition}{Proposition}

\newcommand{\mysec}[1]{Section~\ref{sec:#1}}
\newcommand{\eq}[1]{Eq.~(\ref{eq:#1})}
\newcommand{\myfig}[1]{Figure~\ref{fig:#1}}

 \def \ds { \displaystyle}

\def \S{  { \mathbb{S}} }

\def \E{{\mathbb E}}

\def \B{{\mathbb B}}

\title{On the relationship between multivariate splines \\
and infinitely-wide neural networks}

\author{Francis Bach\\
Inria,  Ecole Normale Sup\'erieure \\
PSL Research University \\
\url{francis.bach@inria.fr}}

\date{\today}

\begin{document}
\maketitle

\begin{abstract}
 We consider multivariate splines and show that they have a random feature expansion as infinitely wide neural networks with one-hidden layer and a homogeneous activation function which is the power of the rectified linear unit. We show that the associated function space is a Sobolev space on a Euclidean ball, with an explicit bound on the norms of derivatives.
This link provides a new random feature expansion for multivariate splines that allow efficient algorithms. This random feature expansion is numerically better behaved than usual random Fourier features, both in theory and practice. In particular, in dimension one, we compare the associated leverage scores to compare the two random expansions and show a better scaling for the neural network expansion.
\end{abstract}

\section{Introduction}

Multivariate non-parametric regression can be approached from a variety of methods: decision trees, local averaging methods such as Nadaraya-Watson estimation or $k$-nearest-neighbor regression, neural networks, and methods based on positive definite kernels such as smoothing splines, kriging, and kernel ridge regression (see, e.g.,~\cite{hastie2009elements,wasserman2006all,gyorfi2002distribution}).

In this paper, we build on the following known relationship between kernel-based methods and infinitely-wide neural networks~\cite{neal1995bayesian,rahimi2008random}. We consider an activation function $\sigma: \rb \to \rb$, and a one-hidden-layer neural network model on $\rb^d$ of the form 
$$
f(x) = \sum_{j=1}^m \eta_j \sigma(w_j^\top x +b_j),
$$
where $\eta_j \in \rb$, $(w_j,b_j) \in \rb^{d+1}$, for $j=1,\dots,m$.
When an $\ell_2$-regularization is added to the objective function which is used to fit the model,   this is equivalent to using a kernel-based method with positive-definite kernel 
\BEQ
\label{eq:eqk}
\hat{k}(x,y) = \frac{1}{m} \sum_{j=1}^m \sigma(w_j^\top x +b_j) \sigma(w_j^\top y +b_j).
\EEQ
See, e.g.,~\cite{scholkopf2002learning} for an introduction to kernel methods.
If the $m$ input weights $(w_j,b_j) \in \rb^{d+1}$ are sampled independently and identically distributed, when $m$ tends to infinity, by the law of large numbers, $\hat{k}(x,y)$ tends to the equivalent kernel 
\BEQ
\label{eq:k}
k(x,y) = \E_{(w,b)} \big[  \sigma(w^\top x +b) \sigma(w^\top y +b) \big] .
\EEQ
This equivalence between infinitely wide neural networks and kernel methods has already been used in several ways:
\BIT
\item Given a known kernel $k$ which can be expressed as an expectation as in \eq{k}, we can use the approximate kernel $\hat{k}$ as in \eq{eqk}, and its explicit random features to derive efficient algorithms~\cite{rahimi2008random,rudi2017generalization}: with $n$ observations, we can circumvent the computation of the $n\times n$ kernel matrix by computing the $m$-dimensional feature vector for each of the $n$ observations, which is advantageous when $m < n$.
\item Given a known neural network architecture, they allow the study of the regularization properties of using over-parameterized models, that is, with the number of hidden neurons going to infinity~\cite{le2007continuous}.
\EIT

In this paper, we make the following contributions, that contribute to the two ways mentioned above of relating kernels and neural networks:
\BIT
\item We consider multivariate splines~\cite{duchon1977splines,buhmann2003radial}, with kernels proportional to $\| x - y\|_2^{2\alpha+1}$, for $\alpha \in \mathbb{N}$ (where $\| \cdot\|_2$ denotes the standard Euclidean norm), and show that they have a random feature expansion as infinitely wide neural networks with one-hidden layer and a homogeneous activation function which is the $\alpha$-th power of the ``rectified linear unit''~\cite{nair2010rectified}. This extends the earlier work of~\cite{le2007continuous}, which proved this link for $\alpha=0$ (step activation function).
\item We show that the associated function space is a Sobolev space with order $s = \frac{d+1}{2} + \alpha$, with an explicit dependence between the norms.
\item This link provides a new random feature expansion for multivariate splines that allow efficient algorithms. This random feature expansion numerically behaves better than usual random Fourier features, both in theory and practice. In particular, in dimension one, we compare the associated leverage scores~\cite{bach2017equivalence} to compare the two random expansions. This also provides a more efficient alternative to random Fourier feature expansion for all Mat\'ern kernels~\cite[page 84]{williams2006gaussian}.
\EIT

\section{From one-hidden layer neural networks to positive-definite kernels}
We consider the Euclidean ball $\B^d(R)$ of  center $0$ and radius $R$ in  $\rb^d$, for $R>0$, and consider the estimation of real-valued functions on $\B^d(R)$.

For $\alpha \in \mathbb{N}$, we consider activation functions $\sigma$ of the form $\sigma(u) = (u_+)^\alpha
= \max\{u,0\}^\alpha$, that is, $\sigma(u) = 0$ for $u \leqslant 0$, and $\sigma(u) = u^\alpha$ for $u>0$, with the usual convention that $u^0 = 1$ if $u>0$, with a particular focus on $\alpha \in \{0,1\}$. For $\alpha=0$, we recover the step-function $\sigma(u)=1_{u>0}$, and for $\alpha=1$ the rectified linear unit $\sigma(u) = u_+$.

We consider randomly distributed weights $(w,b) \in \rb^{d+1}$, and the positive definite kernel
$$
k(x,y) = \E_{(w,b)} \big[ \sigma(w^\top x + b) \sigma(w^\top y + b)\big].
$$

Since we use homogeneous activation functions, we can normalize weights $(w,b)$ so that they have compact supports.
Several normalizations and distributions can be used to obtain closed-form formulas.

\paragraph{Full spherical symmetry in $\rb^{d+1}$.}
The first normalization is to write 
$$w^\top x + b = { w \choose b/R} ^\top { x \choose R},$$
and let ${ w \choose b/R}$ be rotationally invariant, for example, be uniformly distributed on the unit $\ell_2$-sphere in dimension $\rb^{d+1}$. This leads to closed-form formulas~\cite{cho2009kernel,bach2017breaking} for the corresponding kernel $\tilde{k}^{(\alpha)}_d$, with $\ds\cos \varphi = \frac{ x^\top y + R^2}{  ( \|x\|_2^2 + R^2)^{1/2}  ( \|y\|_2^2 + R^2)^{1/2}} \geqslant 0$ (so that $\varphi \in [0,\pi/2]$), leading to for  small values of $\alpha$:
\BEAS
\tilde{k}^{(0)}_d(x,y) & = &  \frac{1}{2\pi} ( \pi - \varphi )
\\
\tilde{k}^{(1)}_d(x,y) & = &  \frac{ 1 }{2\pi(d+1)}  ( \|x\|_2^2 + R^2)^{1/2}  ( \|y\|_2^2 + R^2)^{1/2} \times \big[
\sin \varphi + ( \pi - \varphi)  \cos \varphi
\big]
\\
\tilde{k}^{(2)}_d(x,y) & = &   \frac{1 }{2\pi (d+1)(d+3)}   ( \|x\|_2^2 + R^2)  ( \|y\|_2^2 + R^2) 
 \times \big[ 3 \sin \varphi \cos \varphi + (\pi - \varphi)  ( 1 + 2 \cos^2\varphi) \big].
\EEAS
More generally (see~\cite{cho2009kernel}), we have:
$$
\tilde{k}^{(\alpha)}_d(x,y) = \frac{1 }{2\pi (d+1)(d+3) \cdots (d+2\alpha-1) }   ( \|x\|_2^2 + R^2)^{\alpha/2}  ( \|y\|_2^2 + R^2) ^{\alpha/2} 
 \times J_\alpha( \varphi),
$$
with $J_\alpha(\varphi) = (-1)^{\alpha} (\sin \varphi)^{2\alpha+1} \Big( \frac{1}{\sin \varphi} \frac{d}{d\varphi} \Big)^\alpha \Big( \frac{\pi - \varphi}{\sin \varphi} \Big)$, which is of the form $P_\alpha(\cos \varphi, \sin \varphi) + Q_\alpha(\cos \varphi, \sin \varphi) ( \pi - \varphi)$ for $P_\alpha$ and $Q_\alpha$ polynomials of degree less than $\alpha$.
 
\paragraph{Regularization properties.}
The associated space of functions can be described through spherical harmonics in ambient dimension $d+1$, as, e.g.,~described in  \cite{pmlr-v130-scetbon21b,bach2017breaking}, through the use of the Laplacian on the hyper-sphere. This requires, however, strong knowledge of spherical harmonics and is not easy to relate to classical notions of derivatives in $\rb^d$. Note that Hermite polynomials can be used as well~\cite{daniely2016toward}. Overall we obtain the Sobolev space with degree $s = \frac{d+1}{2}+   \alpha$ over $\B^d(R)$, but with a non-explicit expression in terms of derivatives.

In this paper, we consider another normalization with easier interpretations and links with existing kernels from the statistical literature. This is done using only a spherical symmetry on $w \in \rb^d$.

\paragraph{Partial spherical symmetry (in $\rb^{d})$.}
We can instead choose to have ${ w \choose b/R}$ uniformly distributed on the product $\mathbb{S}^{d-1} \times [-1,1]$ where $\mathbb{S}^{d-1} \subset \rb^d$ is the unit $\ell_2$-sphere, following~\cite{le2007continuous} that introduced this normalization for $\alpha=0$.  This corresponds to $w$ uniform on the sphere $\mathbb{S}^{d-1}$ and $b$ uniform on $[-R,R]$. 

The main goal of this paper is to provide closed-form formulas for the kernel as well as to study the regularization properties. We will start in dimension one ($d=1$) and extend to all dimensions in later sections.

We thus define the positive-definite kernel:
$$
k_d^{(\alpha)}(x,y) = \E_{(w,b)}  \big[ (w^\top x + b)_+^\alpha  (w^\top y + b)_+^\alpha\big],
$$
where $(w,b)$ is uniform on $\S^{d-1} \times [-R,R]$.

\section{Kernels on the interval $[-R,R]$ ($d=1$)}

\label{sec:kernelinterval}

We first consider the kernel for $d=1$, where $w \in \{-1,1\}$, for which we have, by a change of variable $b \to -b$:
\BEAS
 k_1^{(\alpha)}(x,y) & = &  \frac{1}{4R} \int_{-R}^R ( x+b)_+^\alpha( y+b)^\alpha_+ db
+
\frac{1}{4R} \int_{-R}^R ( -x+b)_+^\alpha( -y+b)^\alpha_+ db \\
 & = &  \frac{1}{4R} \int_{-R}^R ( x-b)_+^\alpha( y-b)^\alpha_+ db
+
\frac{1}{4R} \int_{-R}^R ( b-x)_+^\alpha(b-y)^\alpha_+ db.
\EEAS

\subsection{Closed-form formulas}
We first consider the case $\alpha=0$ and then generalize from it. We have for $\alpha=0$, by direct integration:
\BEAS
 k_1^{(0)}(x,y) & = &  \frac{1}{4R} \int_{-R}^ { \min\{x,y\}}    db
+
\frac{1}{4R} \int_{\max\{x,y\}} ^R db  = \frac{1}{2} - \frac{1}{4R} \big[ \max\{x,y\} - \min\{x,y\} \big]\\
& = & 
 \frac{1}{2}- \frac{1}{4R}|x-y| .
 \EEAS

A more tedious direct computation gives the expression for other small values of $\alpha$, as:
\BEAS
k_1^{(1)}(x,y) & = &    \frac{R^2}{6} + \frac{1}{2} x y +  \frac{1}{24 R}|x-y|^3\\
k_1^{(2)}(x,y)   & = & \frac{R^4}{10} + \frac{2 R^2 xy}{3} + \frac{R^2}{6} ( x^2 + y^2)  + \frac{1}{2} x^2 y^2 - \frac{1}{120 R} | x - y|^5 .
\EEAS
This can be extended to all values in $\alpha$ in the following proposition, shown in Appendix~\ref{app:proofd1}.
Note that in one dimension, \cite{kristiadi2021infinite} already made the connection between cubic splines and infinitely-wide neural networks.

\begin{proposition}[Closed-form formula for $d=1$] 
\label{prop:d1}
Let $\alpha \in \mathbb{N}$, we have, for  $(w,b)$ uniformly distributed on the product $\{-1,1\} \times [-R,R]$, 
$$
k_1^{(\alpha)}(x,y) = \E_{(w,b)}  \big[ (w^\top x + b)_+^\alpha  (w^\top y + b)_+^\alpha\big]
= P^{(\alpha)}_1(x^2, y^2, x y) + \frac{1}{R} c^{(\alpha)}_1|x-y|^{2\alpha+1},
$$
where $P^{(\alpha)}_1$ is a polynomial of degree $\alpha$, such that  $k_1^{(\alpha),({\rm pol})}(x,y) =  P^{(\alpha)}_1(x^2, y^2, x y) $ is a positive-definite kernel, and
$\ds c^{(\alpha)}_1 = \frac{(-1)^{\alpha+1}}{4} \frac{ ( \alpha! )^2}{(2\alpha+1)! }$.
\end{proposition}

Moreover, as shown in Appendix~\ref{app:proofd1}, we have a special form of polynomial kernel
$k_1^{(\alpha),({\rm pol})}(x,y) =  P^{(\alpha)}_1(x^2, y^2, x y) $, as:
\BEA
\notag k_1^{(\alpha),({\rm pol})}(x,y) & = &   \frac{1}{4R}\int_{-R}^R    ( x-b)^\alpha(y-b)^\alpha db , \ \mbox{ which can be expressed as} \\
\notag k_1^{(\alpha),({\rm pol})}(x,y)  & = & \frac{1}{2R} \sum_{i,j=0}^{\alpha} 1_{i+j\  { \rm even }} \cdot { \alpha \choose i} {\alpha \choose j}  {x^i y^j}
\frac{R^{2\alpha +1 - i -j }}{2\alpha +1 - i -j}
\\
\label{eq:expl} & = & \frac{1}{2 } \sum_{s=0}^{\alpha} \frac{R^{2\alpha - 2s }}{2\alpha +1 - 2s }
\sum_{i,j=0}^{\alpha} 1_{i+j = 2s} \cdot { \alpha \choose i} {\alpha \choose j}  {x^i y^j}.
\EEA
Note that the term for $s=0$, is $\frac{R^{2\alpha}}{2(2\alpha+1)}$, while for $\alpha>0$, the term corresponding to $s=1$ is equal to $\frac{\alpha^2 R^{2(\alpha-1)}}{2(2\alpha-1)} xy$. 
In all cases, it can be computed in time at most $O(\alpha^2)$. Moreover, the corresponding feature space leads to all polynomials of degree less than $\alpha$ (see proof in Appendix~\ref{app:pol}).
This result will be directly extended to dimensions $d$ greater than one in Prop.~\ref{prop:closedformd}.

\subsection{Corresponding norm}
All positive-definite kernels define a Hilbert space of real-valued functions on $\B^d(R)$ with a particular norm. For kernels that can be expressed as expectations, this norm $\Omega_1^{(\alpha)}$ is equal to~\cite{berlinet2011reproducing,bach2017breaking}:
\BEAS
\Omega_1^{(\alpha)}(f)^2 & = & \inf_{ \eta_{\pm}: [-R,R] \to \rb } \frac{1}{4R} \int_{-R}^R \big[ \eta_+(b)^2
+ \eta_-(b)^2\big] db \\
& &  \mbox{ such that } \forall x \in [-R,R], f(x) = \frac{1}{4R} \int_{-R}^R \big[ \eta_+(b) (x-b)_+^\alpha 
+\eta_-(b) (b-x)_+^\alpha  \big]   db,
\EEAS
where the infimum is taken over square-integrable functions $\eta_+$ and $\eta_-$.

\paragraph{Special case $\alpha=0$.} For $f$ continuously differentiable, we can  use and average  the  two simple representations:
$$
f(x)  = f(-R) + \int_{-R}^x f'(b) (x-b)_+^0 db = f(R) - \int_{-R}^R f'(b) (b-x)_+^0 db,
$$
to get $ \ds f(x) = \frac{1}{2}[ f(R)+f(-R) ] + 2R\int_{-R}^R f'(b) (x-b)_+^0 \frac{db}{4R}
- 2R\int_{-R}^R f'(b) (b-x)_+^0 \frac{db}{4R}.$
The constant function equal to $1/2$ on $[-R,R]$ can be represented as:
$$
\int_{-R}^R (x-b)_+^0 \frac{db}{4R} + \int_{-R}^R (b-x)_+^0 \frac{db}{4R}.
$$
We can thus take: $\eta_+(b) = 2Rf'(b) +  [ f(R)+f(-R) ] $
and $\eta_+(b) = -2Rf'(b) +  [ f(R)+f(-R) ] $.

This leads to the squared norm $\Omega_1^{(0)}(f)^2$ less than (since the cross-terms cancel):
$$
2R  \int_{-R}^R f'(x)^2 dx+  \big[ f(-R)+f(R)\big]^2.
$$
In particular, the norm is finite as soon as the quantity above is well-defined, that is, $f'$ square integrable. To show that this is indeed the correct norm, we simply need to check that our representation is optimal, which is shown below for all $\alpha$'s (see Prop.~\ref{prop:d1norm}). Thus 
$$
\Omega_1^{(0)}(f)^2 = 2R  \int_{-R}^R f'(x)^2 dx+  \big[ f(-R)+f(R)\big]^2.
$$

 \paragraph{General case $\alpha \geqslant 0$.}

To obtain the norm, we can notice that continuous expansions with functions $(x-b)_+^\alpha$ are exactly obtained from Taylor expansions with integral remainders, which apply to functions defined on $[-R,R]$ with $\alpha+1$ continuous derivatives:
$$
f(x) = \sum_{i=0}^\alpha \frac{ f^{(i)}(-R)}{i!} (x+R)^i + \int_{-R}^R \frac{ f^{(\alpha+1)}(b)}{\alpha!} (x-b)_+^\alpha db.
$$
Ignoring the boundary conditions, we see that $\eta_+(b)$ should be related to $\frac{1}{\alpha!} f^{(\alpha+1)}(b) $, and that the RKHS norm should include the integral $\ds \int_{R}^R  f^{(\alpha+1)}(x)^2 dx$. The following proposition makes this explicit (see proof in Appendix~\ref{app:proofd1norm}).

\begin{proposition}[RKHS norm for $d=1$] 
\label{prop:d1norm}
The RKHS norm on functions on $[-R,R]$ associated to the kernel~$k_1^{(\alpha)}$ is equal to:
$$\Omega_1^{(\alpha)}(f)^2 =  \frac{2R}{\alpha!^2} \int_{-R}^R f^{(\alpha+1)}(x)^2 dx + \Theta^{(\alpha)}\big[ f^{(0)}(-R),f^{(0)}(R),\dots,
f^{(\alpha)}(-R),f^{(\alpha)}(R)\big],
$$
where $ \Theta^{(\alpha)}$ is non-negative quadratic form.
\end{proposition}
For example, for  $\alpha=1$, we get:
\BEAS
 \Omega_1^{(1)}(f)^2 & = & 2R \int_{-R}^R f''(x)^2 dx + \big[ f'(R)+f'(-R) \big]^2
+ \frac{3}{2R^2} \big[ f(-R) + f(R)  - R   f'(R)  +R f'(-R) \big]^2 .\EEAS

\paragraph{Equivalence to classical Sobolev norms.} Using classical results on Sobolev spaces~\cite{adams2003sobolev}, the norm in Proposition~\ref{prop:d1norm} can be shown to be equivalent to the classical squared Sobolev norm $ \ds R \int_{-R}^R f^{(\alpha+1)}(x)^2 dx + 
 \frac{1}{R^{2\alpha+1} } \int_{-R}^R f(x)^2 dx$.
We will generalize this to all dimensions and provide an explicit equivalence in the following sections.

\section{Kernels on the ball $\B^d(R)$ ($d\geqslant 1$)}

We now extend results from \mysec{kernelinterval} to all dimensions $d\geqslant 1$. We will get explicit closed-form formulas but with a slightly less explicit formulation for the RKHS norm.

 \subsection{Closed-form formulas}
	We start with the closed-form formula that directly extends Prop.~\ref{prop:d1}.
\begin{proposition}[Closed-form formula for $d\geqslant 1$] 
\label{prop:closedformd}
 Let $\alpha \in \mathbb{N}$, we have, for  ${ w \choose b/R}$ uniformly distributed on the product $\mathbb{S}^{d-1} \times [-1,1]$, 
$$
k_d^{(\alpha)}(x,y) = \E_{(w,b)}  \big[ (w^\top x + b)_+^\alpha  (w^\top y + b)_+^\alpha\big]
= P_d^{(\alpha)}(\| x\|_2^2, \|y\|_2^2, x^\top y) + \frac{1}{R} c_d^{(\alpha)}
\|x-y\|_2^{2\alpha+1},
$$
where $P_d^{(\alpha)}$  is a polynomial of degree $\alpha$, such that 
$ k_d^{(\alpha),({\rm pol})}(x,y) =  P_d^{(\alpha)}(\| x\|_2^2, \|y\|_2^2, x^\top y) $ is a positive-definite kernel, and
$\ds c^{(\alpha)}_d =       
\frac{(-1)^{\alpha+1}}{4 \sqrt{\pi}} \frac{ \alpha!^3 \Gamma(\frac{d}{2})}{ (2\alpha + 1)! \,   \Gamma(\frac{d}{2}+ \frac{1}{2} + \alpha)}$.
\end{proposition}
\begin{proof}
We have $\ds k_d^{(\alpha)}(x,y) = \E_{w}  \big[   k_1^{(\alpha)}(w^\top x,w^\top y)  \big]$, and we
 simply use, for $w$ uniform on the sphere: $\ds \E [ |w^\top z |^{2\alpha+1} ]
=  \| z \|_2^{2\alpha+1}  \frac{ \Gamma(1+\alpha) \Gamma(\frac{d}{2})}{\Gamma(\frac{1}{2}) \Gamma(\frac{d}{2}+ \frac{1}{2} + \alpha)}$ (see Appendix~\ref{app:formulas}), which leads to the expression for $c_d^{(\alpha)}$. To treat the polynomial kernel part, we use \eq{expl}, and the fact that for $w$ uniform, and $i+j$ even,
$\E \big[ (w^\top x)^i (w^\top y)^j \big]$ is a polynomial of degree less than $i+j$ in $x^\top x$, $y^\top y$ and $y^\top x$.
\end{proof}
Like for $d=1$, we have an integral representation for the kernel 
$P_d^{(\alpha)}(\| x\|_2^2, \|y\|_2^2, x^\top y) = k_d^{(\alpha),({\rm pol})}(x,y)$, as
$$  k_d^{(\alpha),({\rm pol})}(x,y) =    \frac{1}{4R}\int_{-R}^R  \E_{w} \big[   ( w^\top x+b)^\alpha(w^\top y +b)^\alpha \big] db.$$
Note that we have defined a new positive-definite polynomial kernel, which is an alternative to the standard kernel $(x,y) \mapsto ( R + x^\top y)^{\alpha}$, that can be computed in time $O(d)$ (with a constant that depends on $\alpha$). As shown in Appendix~\ref{app:pold}, the corresponding space spans all polynomials of degree less than $\alpha$ (or equal).

 We have, for $\alpha \in \{0,1,2\}$:
\BEAS
k_d^{(0)}(x,y)
& = &  \frac{1}{2}
- \frac{1}{4R} \frac{ \Gamma(1) \Gamma(\frac{d}{2})}{\Gamma(1/2) \Gamma(\frac{d+1}{2})} \| x - y \|_2\\
k_d^{(1)}(x,y)& = &
\frac{R^2}{6}
+ \frac{1}{2d} x^\top y      + \frac{1}{24R} \frac{ \Gamma(2) \Gamma(\frac{d}{2})}{\Gamma(\frac{1}{2}) \Gamma(\frac{d}{2}+ \frac{3}{2})} \| x - y \|_2^3
\\
k_d^{(2)}(x,y)& = &
\frac{R^4}{10}
+ \frac{1}{3d} x^\top y  + \frac{R^2}{6d} ( \|x\|_2^2 + \|y\|_2^2)
+ \frac{1}{2d(d+2)} ( 2 ( x^\top y)^2 + \|x\|_2^2 \|y\|_2^2 ) \\
& & \hspace*{8.5cm}   - \frac{1}{120R} \frac{ \Gamma(2) \Gamma(\frac{d}{3})}{\Gamma(\frac{1}{2}) \Gamma(\frac{d}{2}+ \frac{5}{2})} \| x - y \|_2^5.
\EEAS

\paragraph{A simple bound.}
We will need to provide a bound on the associated features. We have, for $\|x\|_2 \leqslant R$:
\BEA
\notag k_d^{(\alpha}(x,x) & = & \frac{1}{4R}\int_{-R}^R  \E_{w} \big[   ( w^\top x+b)^2 \big] \alpha  db
\\ 
\label{eq:CC}& \leqslant & \frac{1}{4R}\int_{-R}^R  \E_{w}   ( 2R )^{2 \alpha}  db = \frac{1}{2} (2R)^{2\alpha}.
\EEA

\subsection{Corresponding norms}
 In dimension $d=1$, we could give an explicit formula for the corresponding RKHS norm, which relied on Taylor's formula with integral remainders.   This will be less explicit in higher dimensions, and we will need to use the theory of multivariate splines~\cite{duchon1977splines,buhmann2003radial}.
 
\section{Link with multivariate splines}
In this section, we first review splines and then draw explicit links. For more details on multivariate splines, see~\cite{wendland2004scattered,wahba1990spline,berlinet2011reproducing}.

\subsection{Review of multivariate splines}
\label{sec:splines}
We consider the function:
$$
E_d^{(\alpha)}(z) = c_d^{(\alpha)} \| z\|_2^\alpha,
$$
which has Fourier transform (defined as a distribution, see~\cite{friedlander1998introduction}):
$$
c_d^{(\alpha)}  (-1)^{\alpha+1} {2^{d+1+2\alpha}\pi^{d/2-1}\Gamma(\alpha+3/2) \Gamma(d/2+1/2+\alpha)} 
\frac{1}{\|\omega\|_2^{d+1+2\alpha}} =
b_d^{(\alpha)} \frac{1}{\|\omega\|_2^{d+1+2\alpha}} .
$$

The kernel $E_d^{(\alpha)}  (x-y)$ is known to be ``conditionally positive of order $\alpha$''~\cite{wendland2004scattered,wahba1990spline}, that is for each $x_1,\dots,x_n \in \rb^d$, and $\lambda_1,\dots,\lambda_n$ such that $\sum_{i=1}^n \lambda_i P(x_i) = 0$ for all polynomials $P$ of degree less than $\alpha$, 
$$\sum_{i,j=1}^n \lambda_i \lambda_j E_d^{(\alpha)}  (x_i - x_j) \geqslant 0.$$
We also know that for any function $L: \rb^n \to \rb \cup \{ +\infty\}$, the minimization of 
$$L(f(x_1),\dots,f(x_n)) + \frac{1}{2 b_d^{(\alpha)} } \frac{1}{(2\pi)^d} \int_{\rb^d} | \hat{f}(\omega)|^2 \|\omega\|_2^{d+1+2\alpha} d\omega$$
is attained at $\ds f(x) = P(x) + \sum_{i=1}^n \lambda_i E_d^{(\alpha)}(x-x_i)$, with $P$ and $\lambda$ obtained through the minimization of
$$L(f(x_1),\dots,f(x_n)) + \frac{1}{2} \sum_{i,j=1}^n \lambda_i \lambda_j E_d^{(\alpha)}  (x_i - x_j) $$
with respect to the polynomial $P$ of degree less than $\alpha$,  and $\lambda \in \rb^n$ such that  
 $\sum_{i=1}^n \lambda_i Q(x_i) = 0$ for all polynomials $Q$ of degree less than $\alpha$~\cite{duchon1977splines}. 
 
When $d$ is odd, then we have an explicit representation in terms of partial derivatives:
$$
\frac{1}{(2\pi)^d} \int_{\rb^d} | \hat{f}(\omega)|^2 \|\omega\|_2^{d+1+2\alpha} d\omega
= \int_{\rb^d} \| D^{\frac{d+1}{2}+\alpha}  f (x) \|^2dx,
$$
where $D^{\frac{d+1}{2}+\alpha}  f $ is the tensor of all partial derivatives of order $ \frac{d+1}{2}+\alpha$.

The expansion above can be extended to the representation of functions on $\B^d(R)$ such that the norm above is finite as $f(x) = \ds P(x) + \int_{\B^d(R)} \!\! E_d^{(\alpha)}(x-y) d\lambda (y)$ where $\lambda$ is Radon measure.

\paragraph{Algorithms.} Given a map $\varphi: \rb^d \to \rb^{m}$ that can represent all polynomials of degree less than $\alpha$, that is, with $m = { d + \alpha \choose \alpha}$, then we can write the vector $y \in \rb^n$ defined as $y_i = f(x_i)$, as $y = \Phi \nu + K \lambda$, with $\Phi \in \rb^{n \times m}$ the matrix with all $\varphi(x_i)$, $i=1,\dots,n$, and $K \in \rb^{n \times n}$ the kernel matrix associated with $E_d^{(\alpha)}$.
We add the constraint $\Phi^\top \lambda = 0$. Being conditionally positive means that for $\rho$ large enough, $K + \rho \Phi \Phi^\top$ is positive semi-definite. In this paper, we provide an explicit $\rho \Phi \Phi^\top$ that makes this happen when the data are constrained in $\B^d(R)$.
Note that when  $\Phi^\top \lambda = 0$, the polynomial part of our kernel becomes irrelevant.

The algorithm above requires solving an optimization problem of dimension $n$, while we will see below how this can be reduced using random features.

\subsection{Equivalence with Sobolev space}

As shown in Prop.~\ref{prop:closedformd}, our kernel is equal to
$$
k_d^{(\alpha)}(x,y) = k_d^{(\alpha),({\rm pol})}(x,y) +c_d^{(\alpha)} \| x- y \|_2^{2\alpha+1}.
$$
We will show that the RKHS norm is equivalent to $\Omega_d^{(\alpha)({\rm eq})}(f)^2$ defined as the minimal value of
\BEQ
\label{eq:Aeq}
  \frac{1}{  R  } \frac{1}{b_d^{(\alpha)} (2\pi)^d}   \int_{\rb^d} \| \omega\|_2^{d+1+2\alpha} | \hat{g}(\omega)|^2 d\omega + 
\frac{\ws^2}{R^{2\alpha}} \int_{\B^d(R)} |g(x)|^2 dx,
\EEQ
over all functions $g: \rb^d \to \rb$ that is equal to $f$ on $\B^d(R)$, for a well-chosen constant $\ws$.

\paragraph{Norm comparisons.} We consider two positive constants $\ws$ and $\bs$ such that:
\BIT
\item For any polynomials $P$ of degree less than $\alpha$, we have
$ \ds \Omega_d^{(\alpha)}(P) \leqslant \frac{\ws}{ R^\alpha} \| P \|_{L_2(\B^d(R))}$ for a positive constant $\ws$. Such a constant exists because two kernels defining a norm on the finite-dimensional space must have equivalent norms.
We currently do not have an explicit upper bound on the constant~$\ws$.

\item For any $f$ in the RKHS defined by $k_d^{(\alpha)}$, 
$ \ds  \| f \|_{L_2(\B^d(R))} \leqslant R^\alpha \bs \Omega_d^{(\alpha)}(f) $. This has to exist because, for any $x \in \B^d(R)$, we have, like for all RKHSs, $f(x)^2 \leqslant  k_d^{(\alpha)}(x,x) \cdot \Omega_d^{(\alpha)}(f)^2$. We thus have, by integration,
 $\bs^2 R^{2\alpha} \leqslant {\rm vol}(\B^d(R))  \sup_{x \in \B^d(R)} k_d^{(\alpha)}(x,x) \leqslant {\rm vol}(\B^d(R)) \frac{1}{2} ( 2R)^{2\alpha} $ by \eq{CC}, and thus 
$ \bs^2 \leqslant {\rm vol}(\B^d(R)) 2^{2\alpha-1}$.

\EIT
 Note that we must have $\ws \bs \geqslant 1$. We now prove the equivalence.
 
 \begin{proposition}[RKHS norm for $d\geqslant 1$]
 For $f \in L_2(\B^d(R))$, we have:
 $$
 \frac{1}{  \ws  \bs \sqrt{2} } \Omega_d^{(\alpha)({\rm eq})}(f)
\leqslant \Omega_d^{(\alpha)}(f)  \leqslant  2\sqrt{2} \, \ws  \bs  \, \Omega_d^{(\alpha)({\rm eq})}(f).
 $$
 \end{proposition}
 \begin{proof}
For the upper-bound, we  consider a function $g$ attaining the minimization problem defining  $\Omega_d^{(\alpha)({\rm eq})}(f)^2$ in \eq{Aeq}.

 From \mysec{splines}, we can express $f$ as $\ds f(x) = \int_{\B^d(R)}k_d^{(\alpha)}(x,y) d \lambda(y) + P(x)$, where $\lambda$ is a Radon measure on $\B^d(R)$, such that $\ds \int_{\B^d(R)} \!\! Q(y) d\lambda(y)=0$ for all polynomials $Q$ of degree less than $\alpha$. Moreover, since we have a minimum norm representation, we get, 
\BEAS
\frac{1}{  b_d^{(\alpha)} } \frac{1}{(2\pi)^d}  \int_{\rb^d} \| \omega\|_2^{d+1+2\alpha} | \hat{g}(\omega)|^2 d\omega &\geqslant & \int_{\B^d(R)} \int_{\B^d(R)}E_d^{(\alpha)}(x-y) d \lambda(y)  d \lambda(x)
 \\ & = &  
  \int_{\B^d(R)} \int_{\B^d(R)}k_d^{(\alpha)}(x,y) d \lambda(y)  d \lambda(x).\EEAS
The last quantity is equal to $\Omega_d^{(\alpha)}(f - P )^2$ because of the reproducing property of kernels.

 The polynomial $P$ can be expressed in the RKHS because of the part $k_d^{(\alpha),({\rm pol})}(x,y)$. Therefore $f$ is in the RKHS.
We thus only need to show that the RKHS norm of $P$ is less than its $L_2$-norm on $\B^d(R)$, since then the RKHS norm of $f$ is less than a constant times $\Omega_d^{(\alpha)({\rm eq})}(f)$.

  The $L_2$-norm of $P$ on $\B^d(R)$ is less than 
the $L_2$-norm of $g$ (which is the one of $f$) plus the $L_2$-norm of the function  $\ds x \mapsto  \int_{\B^d(R)}k_d^{(\alpha)}(x,y) d \lambda(y)$, which is less than $  {\bs}{ R^\alpha}\Omega_d^{(\alpha)}(f-P)$ by definition of $\bs$.

 Thus, since $ \Omega_d^{(\alpha)}(P) \leqslant \frac{\ws}{ R^\alpha} \| P \|_{L_2(\B^d(R))}$ by definition of $\ws$, we get:
\BEAS
\Omega_d^{(\alpha)}(f)
& \leqslant & \Omega_d^{(\alpha)}(P) + \Omega_d^{(\alpha)}(f-P)   
\leqslant  \frac{\ws}{ R^\alpha} \| P \|_{L_2(\B^d(R))}   + \Omega_d^{(\alpha)}(f-P)  \\
& \leqslant & \frac{\ws}{ R^\alpha} \Big(
\| f\|_{L_2(\B^d(R))} +   \bs R^{\alpha} \Omega_d^{(\alpha)}(f-P)  \Big) + \Omega_d^{(\alpha)}(f-P) 
\\
& \leqslant & \frac{\ws}{ R^\alpha} \| f\|_{L_2(\B^d(R))} + \ws \bs \Omega_d^{(\alpha)}(f-P)  + \Omega_d^{(\alpha)}(f-P) 
\\
& \leqslant & 2  \ws \bs     \bigg(\frac{1}{  b_d^{(\alpha)} } \frac{1}{(2\pi)^d R }  \int_{\rb^d} \| \omega\|_2^{d+1+2\alpha} | \hat{f}(\omega)|^2 d\omega\bigg)^{1/2}   + \frac{\ws}{ R^\alpha}  
\| f\|_{L_2(\B^d(R))}  .
\EEAS
Thus $\Omega_d^{(\alpha)}(f)^2 \leqslant  8 \ws^2 \bs^2 \Omega_d^{(\alpha)({\rm eq})}(f)^2$.

For the lower-bound, given $\ds f = \int_{\B^d(R)}  k_d^{(\alpha)}(x,y)d\lambda(y)$ in the RKHS, the extension $g$ that minimizes the squared norm 
$\ds \frac{1}{  b_d^{(\alpha)} } \frac{1}{(2\pi)^d} \int_{\rb^d} \| \omega\|_2^{d+1+2\alpha} | \hat{g}(\omega)|^2 d\omega $, is the one that can be written
$\ds g = \int_{\B^d(R)}  k_d^{(\alpha)}(x,y)d\mu(y)$, with $\mu$ orthogonal to polynomials, and $\ds \int_{\B^d(R))}
\int_{\B^d(R))} k_d^{(\alpha)}(x,y) d\mu(x) d\mu(y)$ minimized. By introducing the measure $\bar{\lambda}$ obtained by projecting on the orthogonal to all polynomials of degree less than $\alpha$, we have:
\BEAS
\int_{\B^d(R))}
\int_{\B^d(R))} k_d^{(\alpha)}(x,y) d\mu(x) d\mu(y)
& \leqslant & 
\int_{\B^d(R))}
\int_{\B^d(R))} k_d^{(\alpha)}(x,y) d\bar{\lambda} (x)  d\bar{\lambda} (y) \\
& \leqslant & \int_{\B^d(R))}
\int_{\B^d(R))} k_d^{(\alpha)}(x,y) d\lambda(x) d\lambda(y),
\EEAS
which is equal to $ \Omega_d^{({\alpha})}(f)^2.$ Thus,   we get:
\BEAS
\Omega_d^{(\alpha)({\rm eq})}(f)^2
& \leqslant & 
 \frac{1}{  b_d^{(\alpha)} } \frac{1}{(2\pi)^d R}  \int_{\rb^d} \| \omega\|_2^{d+1+2\alpha} | \hat{g}(\omega)|^2 d\omega    + \frac{\ws^2}{R^{2\alpha}}
\| g\|_{L_2(\B^d(R))}^2   \\
&
\leqslant &   \Omega_d^{({\alpha})}(f)^2+ \bs^2 \ws^2 \Omega_d^{({\alpha})}(f)^2
\leqslant 2 \bs^2 \ws^2 \Omega_d^{({\alpha})}(f)^2,
\EEAS
which leads to the desired norm equivalence.
  \end{proof}

\subsection{Two competing random feature expansions for $\alpha=0$}

We can first consider the random feature expansion obtained from neural networks in Prop.~\ref{prop:closedformd}, but also a classical one based on the Fourier transform~\cite{rahimi2008random}. We indeed have, for $w$ uniform on the sphere $\S^d$:
 \BEAS
 k_d^{(\alpha)}(x,y) & = &  \E_w \big[ k_1^{(\alpha)}(w^\top x, w^\top y) \big] \\
 & = & k_d^{(\alpha),({\rm pol})}(x,y) + \frac{1}{R} \E_w \big[ c_1^{(\alpha)}| w^\top (x-y) |^{2\alpha+1}\big]
 = \frac{1}{2} - \frac{1}{4R} \E_w \big[ | w^\top (x-y) | \big],
 \EEAS
 for $\alpha=0$.
 Since $|w^\top ( x - y) | \leqslant 2 R$ almost surely (because $x,y \in [-R,R]$), for $\alpha=0$, we can use the one-dimensional Fourier transform of the function:
 $$
 \varphi: u \mapsto  (\frac{1}{2} - \frac{1}{4R} |u| \big) 1_{|u| \leqslant 2R},
 $$
  which is equal to
 $$
 \hat{\varphi}(\omega) =  
 \frac{  \sin^2(R\omega)}{ R \omega^2 } .
 $$
 We thus have, for $\| x-y\|_2 \leqslant 2R$,
 \BEAS
  k_d^{(0)}(x,y) & = &  \E_{w} \Big[
  \frac{1}{2\pi R } \int_{-\infty}^{+\infty}  \frac{  \sin^2(R\tau)}{R  \tau^2 } e^{i\tau w^\top (x-y) } d\tau
  \Big] \\
  & = & \frac{2}{  {\rm vol}(\S^{d-1})}  \frac{1}{2\pi} 
\int_{\S^{d-1}}   \int_{0}^{+\infty}  \frac{  \sin^2(R\tau)}{R \tau^{d+1} } e^{i\tau w^\top (x-y) } \tau^{d-1} d\tau dw 
\\
  & = & \frac{2}{  {\rm vol}(\S^{d-1})}  \frac{1}{2\pi R } 
\int_{\rb^d}   \frac{  \sin^2(R\|\omega\|_2)}{ \|\omega\|_2 ^{d+1} } e^{i\omega^\top (x-y) } d\omega
\mbox{ using the change of variable } \omega = \tau w,
\\
 & = & \frac{  \Gamma(d/2)}{  \pi^{d/2}}  \frac{1}{2\pi R } 
\int_{\rb^d}   \frac{  \sin^2(R\|\omega\|_2)}{ \|\omega\|_2 ^{d+1} } e^{i\omega^\top (x-y) } d\omega\\
 & = &  \frac{1}{(2\pi)^d} 
\int_{\rb^d}   \frac{  \Gamma(d/2) (2\pi)^{d-1}}{  \pi^{d/2} R }   \frac{  \sin^2(R\|\omega\|_2)}{ \|\omega\|_2 ^{d+1} } e^{i\omega^\top (x-y) } d\omega.
  \EEAS
  In other words, the Fourier transform of the function 
  $x \mapsto  (\frac{1}{2} - \frac{1}{R} c_d^{(0)} \|x\|_2 \big) 1_{\|x\|_2 \leqslant 2R}$ is equal
  to $\omega \mapsto \frac{  \Gamma(d/2) (2\pi)^{d-1}}{  \pi^{d/2} R }   \frac{  \sin^2(R\|\omega\|_2)}{ \|\omega\|_2 ^{d+1} }$.
 This leads naturally to the random feature
 $ 
 \cos (\omega^\top x + b)
  $
 with $b$ uniform in $[-\pi,\pi]$ and $\omega$ sampled from the distribution 
with density  $ \ds \frac{2}{(2\pi)^d} 
   \frac{  \Gamma(d/2) (2\pi)^{d-1}}{  \pi^{d/2} R }   \frac{  \sin^2(R\|\omega\|_2)}{ \|\omega\|_2 ^{d+1} } $, which corresponds to $\omega = \tau w$, with $w$ uniform on the sphere and $\tau$ sampled from 
   $\frac{    \sin^2(R\tau)}{ \pi R \tau^2 } $, which can  be done by sampling from 
a Cauchy distribution and using rejection sampling. We can also consider $b$ taking values $0$ and $\pi/2$ uniformly, that is, with random features $\cos(\omega^\top x)$ and $\sin (\omega^\top x)$.

\paragraph{Empirical comparison for $d=1$.}
We compare the two random feature expansions, first visually in \myfig{ex}, then numerically in \myfig{plots}, showing that the random feature expansion based on neural networks has better approximation properties.\footnote{Matlab code to reproduce all figures is available at \url{https://www.di.ens.fr/~fbach/neural_splines_online.zip}.}

\begin{figure}[h]
\begin{center}
\includegraphics[scale=.4]{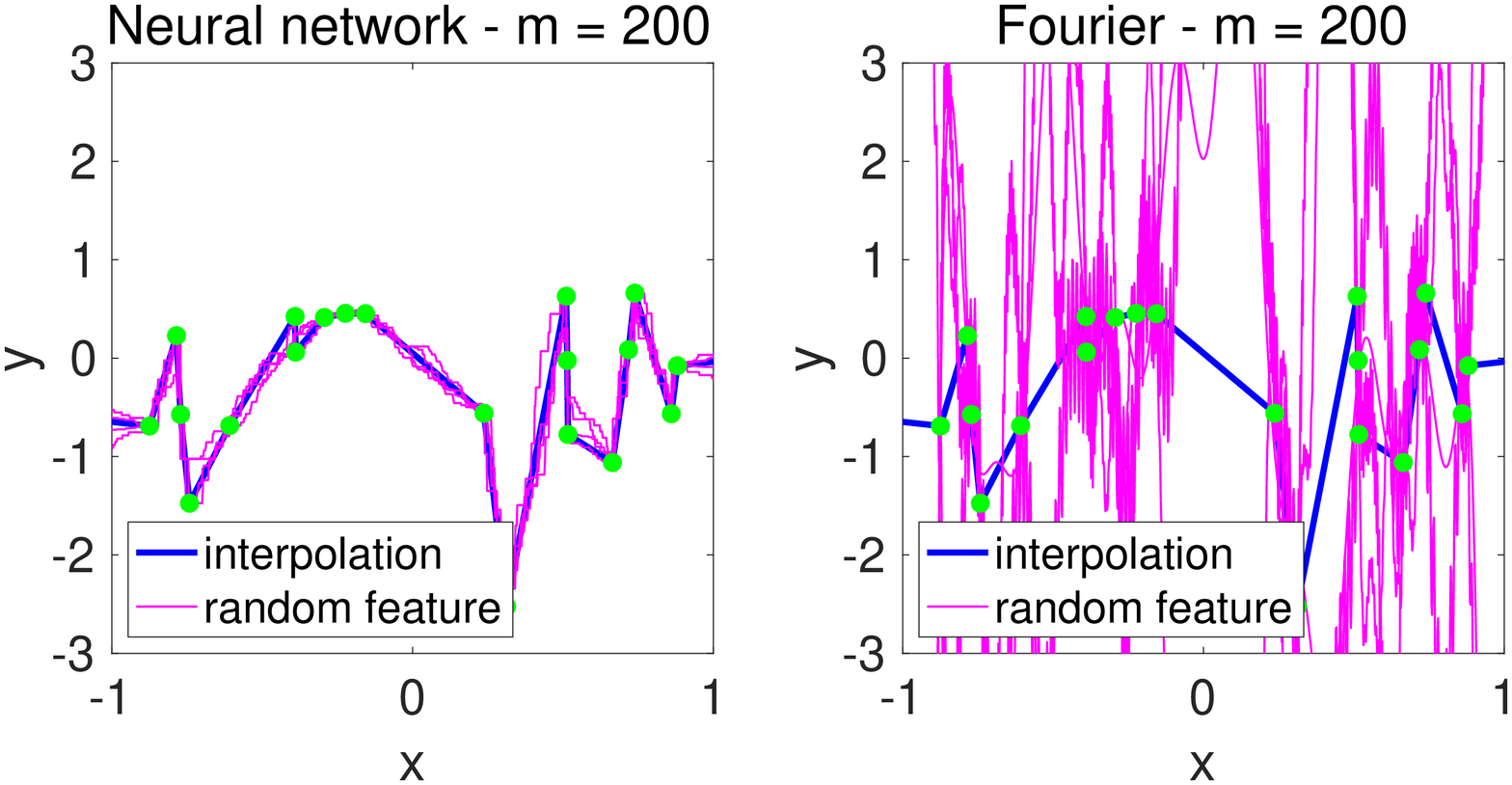}
\end{center}

\vspace*{-.25cm}

\caption{Minimum norm interpolation of the green points by the full RKHS (in blue) and the random feature expansion, for $m=200$. Left: neural network expansion, right: Fourier expansion. Four different draws of the random features are plotted. \label{fig:ex}}
\end{figure}

\begin{figure}[h]
\begin{center}
\includegraphics[scale=.4]{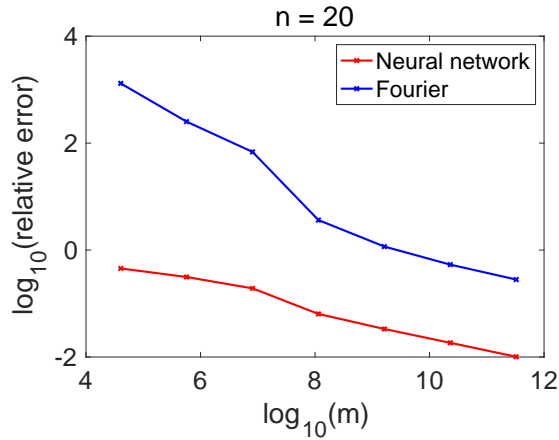}
\end{center}

\vspace*{-.35cm}

\caption{Estimation of the minimum interpolation of the full RKHS by random feature expansions for different values of $m$ (number of random features).  We sample $n=20$ points $x_1,\dots,x_n$ uniformly in $[-1,1]$ as well as $n$ random labels $y_1,\dots,y_n$ from a standard Gaussian distribution. We then compare the minimum interpolation fits with the $L_2([-1,1])$-norm for the full kernel and the random feature approximations. This is averaged over 20 replications for the choice of the input points, and with infinitely many replications for the labels (as the expectation can be taken in closed form): given test points $x'_1,\dots,x'_m$, and the training and testing kernel matrices $K$ and $K'$, together with their approximations $\hat{K}$ and $\hat{K}'$, the error is proportional to $\| K' K^{-1} y - \hat{K}' \hat{K}^{-1} y \|_2^2$, and we can thus compute the expectation with respect to $y$, which is equal to $\| K' K^{-1}   - \hat{K}' \hat{K}^{-1}   \|_F^2$, which is the quantity we plot above. \label{fig:plots}}
\end{figure}

\paragraph{Comparison of leverage scores for $d=1$.}
We want to compare the two random feature expansions, which are of the form
$$
k(x,y) = \E_{v} [ \varphi(x,v)\varphi(y,v)],
$$
for a feature  $\varphi(x,v) \in \rb$ for $v \in \mathcal{V}$. As described in \cite[Section 4]{bach2017equivalence}, to assess the capacity of random feature expansions to approximate the initial function space, a key quantity is the ``leverage score''
$$
v \mapsto \langle \varphi(\cdot,v), ( \Sigma + \lambda \idm)^{-1} \varphi(\cdot,v) \rangle_{L_2(\B^d(R))},
$$
where $\Sigma = \E_{v} \Big[ \varphi(\cdot,v) \otimes_{L_2(\B^d(R))} \varphi(\cdot,v) \Big]$ is an integral operator on ${L_2(\B^d(R))}$. The maximal leverage score over $v \in \mathcal{V}$ has a direct influence on the number of needed random features to get a $\lambda$-approximation in $L_2$-norm of the RKHS ball of the original RKHS: from \cite[Prop.~1]{bach2017equivalence}, up to logarithmic terms, the maximal leverage score is proportional to the number $m$ of necessary random features.

In Appendix~\ref{app:ls}, we compute these leverage scores explicitly for $d=1$, and we see that for the neural network features, the maximal leverage score diverges as $1/\sqrt{\lambda}$ when $\lambda$ tends to zero, while for random Fourier features, they diverge faster as $1/\lambda$, explaining the empirical superiority seen above.

\subsection{A new random feature expansions for  all $\alpha$}
For all $\alpha$, we provided a new kernel $k_d^{(\alpha)}$ that makes the classical multivariate spline positive-definite, together with a random feature expansion, that can be used for efficient estimation. See~\cite{rudi2017generalization} for an analysis.

Other kernels lead to RKHS norms that are equivalent to the same Sobolev norm, such as the Mat\'ern kernels~\cite[page 84]{williams2006gaussian}. These have a natural random Fourier feature expansion (with empirically the same behavior as shown for $\alpha=0$ above), while ours are based on neural networks, with a better behavior when used within random feature expansions.

 \section{Conclusion}
In this paper, we provided new random feature expansions for kernels associated with splines, leading to better properties for Sobolev space on the Euclidean balls than existing expansions based on the Fourier transform. As done by~\cite{bach2017breaking} with feature expansions based on spherical harmonics, this link could be used to provide explicit approximation bounds for neural networks for a large number of neurons (where input weights are also estimated).

\subsubsection*{Acknowledgements}
We thank Nicolas Le Roux and Alessandro Rudi for the interesting discussions related to this work.
 We  acknowledge support from the French government under the management of the Agence Nationale de la Recherche as part of the ``Investissements d’avenir'' program, reference
ANR-19-P3IA-0001 (PRAIRIE 3IA Institute), as well as from the European Research Council
(grant SEQUOIA 724063).

\appendix

\section{A few lemmas about uniform distributions on the sphere }
\label{app:formulas}
If $w$ is uniform on the unit sphere, then:
\BEAS
\E [ |w^\top z |^{2\alpha+1} ]
& = &  \| z \|_2^{2\alpha+1}  \frac{ \Gamma(1+\alpha) \Gamma(\frac{d}{2})}{\Gamma(\frac{1}{2}) \Gamma(\frac{d}{2}+ \frac{1}{2} + \alpha)}  \\
\E [ (w^\top z )^{2\alpha} ] 
& = &  \| z \|_2^{2\alpha} \frac{ \Gamma(\frac{1}{2}+\alpha) \Gamma(\frac{d}{2})}{\Gamma(\frac{1}{2}) \Gamma(\frac{d}{2}+   \alpha)} \\
\E [ (w^\top z )^{2 } ] 
& = &  \| z \|_2^{2 }  / d\\
\E [ z^\top ww^\top t] & = & \frac{1}{d} z^\top t
\\
\E [ ( z^\top ww^\top t )^2] & = & \frac{1}{d(d+2)}\big[ 2  (z^\top t)^2 + z^\top z \cdot t^\top t \big].
\EEAS
This is obtained from $w_1^2$ having a Beta distribution with parameters $(\frac{1}{2},\frac{d-1}{2})$, and using invariance by rotation.

\section{Proof for expressions of RKHS norms, $d=1$}
\label{app:proof_d1}

\subsection{Proof of Proposition~\ref{prop:d1}}
\label{app:proofd1}

 We have
$\ds
\int_{-R}^R ( x-b)_+^\alpha(y-b)^\alpha_+ db = \int_{R}^{\min\{x,y\}}  ( x-b)^\alpha(y-b)^\alpha db,
$
which we can reformulate with $s = \frac{x+y}{2}$ and $\delta = \frac{x-y}{2}$, leading to $x = s + \delta$ and $y = s - \delta$, with $\min\{x,y\} = s - |\delta|$. We get:
\BEAS
\int_{-R}^R ( x-b)_+^\alpha(y-b)^\alpha_+ db & = & \int_{-R}^{s-|\delta|}  ( s+\delta-b)^\alpha(s-\delta-b)^\alpha db 
= \int_{-R}^{s-|\delta|} (   (s-b)^2 - \delta^2 )^\alpha  db
\\
& = & \int_{-R}^{s-|\delta|}   \sum_{i=0}^\alpha { \alpha \choose i}  (s-b)^{2i} (-1)^{i-\alpha}   \delta^{2\alpha - 2i}db
\\
& = &    \sum_{i=0}^\alpha { \alpha \choose i}  \frac{1}{2i+1} \Big[ (R+s)^{2i+1} - |\delta|^{2i+1} \Big]   (-1)^{i-\alpha}   \delta^{2\alpha - 2i}\\
& = &    \sum_{i=0}^\alpha { \alpha \choose i}  \frac{ (-1)^{i-\alpha} }{2i+1}  (R+s)^{2i+1}       \delta^{2\alpha - 2i} 
- \sum_{i=0}^\alpha { \alpha \choose i}  \frac{ (-1)^{i-\alpha} }{2i+1} \times   |\delta|^{2\alpha+1} 
\\
  & = & A_\alpha(x,y) - B_\alpha | x - y|^{2 \alpha+1}     ,
\EEAS
with 
\BEAS
  A_\alpha
  & = & \sum_{i=0}^\alpha { \alpha \choose i}  \frac{ (-1)^{i-\alpha} }{2i+1}  (R+s)^{2i+1}       \delta^{2\alpha - 2i} 
=  \int_{-y}^R  ( x+b)^\alpha(y+b)^\alpha db \\
 B_\alpha & = &  \frac{1}{2^{2\alpha+1}}\sum_{i=0}^\alpha { \alpha \choose i}  \frac{ (-1)^{i-\alpha} }{2i+1}
 = \frac{1}{2^{2\alpha+1}} \int_0^1 \sum_{i=0}^\alpha { \alpha \choose i}  (-1)^{i-\alpha} x^{2i} dx \\
 & = & \frac{(-1)^\alpha}{2^{2\alpha+1}} \int_0^1 ( 1 - x^2)^\alpha dx
 = \frac{(-1)^\alpha}{2^{2\alpha+2}} 2^{2\alpha+1} \int_{0}^1u^\alpha (1-u) ^\alpha du =  \frac{(-1)^\alpha}{2} \frac{ \Gamma(\alpha+1)^2}{\Gamma(2\alpha + 2) },
\EEAS
using the change of variable $\frac{1+x}{2} = u$, $\frac{1-x}{2} = 1 - u$.
This leads to, using symmetries:
\BEAS
k_d^{(\alpha)}(x,y) & = & \frac{1}{4R}\int_{-y}^R    ( x+b)^\alpha(y+b)^\alpha db
+\frac{1}{4R}\int_{y}^R ( -x+b)^\alpha(-y+b)^\alpha db
 \\
 & & \hspace*{8cm} -  \frac{(-1)^\alpha}{4R} \frac{ \Gamma(\alpha+1)^2}{\Gamma(2\alpha + 2) }|x-y|^{2\alpha+1}\\
 & = & \frac{1}{4R}\int_{-y}^R    ( x+b)^\alpha(y+b)^\alpha db
+\frac{1}{4R}\int_{-R}^{-y} ( -x-b)^\alpha(-y-b)^\alpha db
 \\
 & & \hspace*{8cm} -  \frac{(-1)^\alpha}{4R} \frac{ \Gamma(\alpha+1)^2}{\Gamma(2\alpha + 2) }|x-y|^{2\alpha+1}\\
& = & \frac{1}{4R}\int_{-R}^R    ( x+b)^\alpha(y+b)^\alpha db
 -  \frac{(-1)^\alpha}{4R} \frac{ \Gamma(\alpha+1)^2}{\Gamma(2\alpha + 2) }|x-y|^{2\alpha+1}.
 \EEAS
We can then expand using the binomial formula.

\subsection{Proof of Proposition~\ref{prop:d1norm}}
\label{app:proofd1norm}

If we have the representation, for $f$ with $\alpha+1$ continuous derivatives:
$$
\forall x \in [-R,R], f(x) = \frac{1}{4R} \int_{-R}^R \big[ \eta_+(b) (x-b)_+^\alpha 
+\eta_-(b) (b-x)_+^\alpha  \big]   db,
$$
then by taking the $(\alpha+1)$-derivative, we must have:
$$
f^{(\alpha+1)}(x)= 
 \frac{\alpha! }{4R}  \eta_+(x) 
+ \frac{\alpha! }{4R}  (-1)^{\alpha+1} \eta_-(x) .
$$
We thus have:
\BEAS
\eta_+(x) & = & \frac{2R}{\alpha!} f^{(\alpha+1)}(x) + c(x) \\
\eta_-(x) & = & (-1)^{\alpha+1}   \frac{2R}{\alpha!} f^{(\alpha+1)}(x) +  (-1)^{\alpha}  c(x) 
\EEAS
for a certain function $c: [-R,R] \to \rb$.
We have from Taylor formula with integral remainder:
\BEAS
f(x) & \!\! =  \!\! &  \sum_{i=0}^\alpha \frac{ f^{(i)}(-R)}{i!} (x+R)^i + \int_{-R}^R \frac{ f^{(\alpha+1)}(b)}{\alpha!} (x-b)_+^\alpha db \\
f(x) & \!\!  =  \!\! &  \sum_{i=0}^\alpha \frac{ (-1)^i f^{(i)}(R)}{i!} (R-x)^i  - (-1)^{\alpha} \int_{-R}^R\frac{ f^{(\alpha+1)}(b)}{\alpha!} (b-x)_+^\alpha db , \mbox{ and by averaging them},\\
f(x) & \!\!  =  \!\! &  \frac{1}{2} \sum_{i=0}^\alpha \Big[
 \frac{ f^{(i)}(-R)}{i!} (x+R)^i  + 
\frac{ (-1)^i f^{(i)}(R)}{i!} (R-x)^i  
\Big]  \\
& & \hspace*{5cm} + \frac{1}{2}\int_{-R}^R\frac{ f^{(\alpha+1)}(b)}{\alpha!} \big[ (x-b)_+^\alpha -   (-1)^{\alpha} (b-x)_+^\alpha \big] db.
\EEAS
Given our expression for $\eta_+$ and $\eta_-$,  this implies that  for all $x \in [-R,R]$
\BEAS
  \frac{1}{2} \sum_{i=0}^\alpha \Big[
 \frac{ f^{(i)}(-R)}{i!} (x+R)^i  + 
\frac{ (-1)^i f^{(i)}(R)}{i!} (R-x)^i  
\Big] 
& = & \frac{1}{4R} \int_{-R}^R \big[ c(b) (x-b)_+^\alpha +  (-1)^{\alpha}  c(b) (b-x)_+^\alpha ] db \\
& = & \frac{1}{4R} \int_{-R}^R c(b) (x-b)^\alpha db,
\EEAS
leading to constraints on $\ds \int_{-R}^R c(b)b^i$ for $i \in \{0,\dots,\alpha\}$.
The optimal $c$ is obtained by minimizing:
$$
\frac{1}{4R} \int_{-R}^R \big[ \eta_+(b)^2
+ \eta_-(b)^2\big] db = \frac{2R}{\alpha!^2} \int_{-R}^R f^{(\alpha+1)}(b)^2 db + 
\frac{1}{2R} \int_{-R}^R c(b)^2 db.
$$ 
Thus $c$ has to be a polynomial of degree less than $\alpha$, with coefficients which are linear combinations of $f^{(i)}(\pm R)$ for $i \in \{0,\dots,\alpha\}$. This leads to the desired result.

\subsection{Polynomial kernel in one dimension}
\label{app:pol}
Given a polynomial $P$ on $\rb$ of degree less than $\alpha$ (or equal), if we can write it as:
\BEQ
\label{eq:S}
P(x) = \frac{1}{2R} \int_{-R}^R \eta(b) ( x-b)^\alpha db,
\EEQ
then its squared RKHS norm (for $k_1^{(\alpha),({\rm pol})})$ is equal to the infimum of $\ds \frac{1}{4R}  \int_{-R}^R \eta(b)^2db
$. Given the representation in \eq{S}, we have:
$$
P^{(k)}(0) = \frac{1}{2R} \frac{\alpha!}{(\alpha-k)!}\int_{-R}^R \eta(b)   (-b)^{\alpha-k} db
 = \frac{(-1)^{\alpha-k}}{2R} \frac{\alpha!}{(\alpha-k)!}\int_{-R}^R  \eta(b) b^{\alpha-k} db
,$$
which is equal to $\ds \frac{(-1)^{\alpha-k}}{2R} \frac{\alpha!}{(\alpha-k)!} \langle \eta, Q_{\alpha-k} \rangle_{L_2(\B^d(R))}
 $, where $Q_j(b) = b^j$.
 Thus, given that we want to minimize $ \langle \eta, \eta \rangle_{L_2(\B^d(R))}$, the solution has to be a polynomial
 $\eta  = \sum_{j=0}^\alpha s_j  Q_j$, with $s \in \rb^{\alpha+1}$ minimizing 
 $$
 \sum_{i,j=0}^\alpha s_i s_j\langle Q_i, Q_j \rangle_{L_2(\B^d(R))}
 $$
such that $\ds (-1)^{j} 2R   \frac{j!}{\alpha!}  P^{(\alpha-j)}(0) =  \sum_{i=0}^\alpha s_{i}  \langle Q_i, Q_{j} \rangle_{L_2(\B^d(R))}$ for all $j \in \{0,\dots,\alpha\}$. If $P = \sum_{j=0}^\alpha t_j Q_j$, we obtain
$\ds (-1)^{j} 2R   { \alpha \choose j}^{-1} t_{\alpha-j} =  \sum_{i=0}^\alpha s_{i}  \langle Q_i, Q_{j} \rangle_{L_2(\B^d(R))}$.
Since the Gram matrix of the monomials is invertible, the optimal $s$ is a linear function of the coefficients $t$. Thus the norm of $P$ is a positive-definite quadratic form in the coefficients. Hence the norm is equivalent to the $L_2$-norm on the space of polynomials of degree less than $\alpha$ (or equal).

\subsection{Polynomial kernel in  dimension $d \geqslant 1$}
\label{app:pold}
We can apply the same reasoning as in the section above and need to show that we can represent all polynomials of degree less than $\alpha$ as 
$$ P(x) = \frac{1}{2R} \int_{-R}^R \int_{\S^d} \eta(w,b) ( w^\top x + b)^\alpha dw db,$$
for $\eta(w,b)$ square integrable. By taking all partial derivatives at $x=0$, this imposes that all
$$
\int_{-R}^R \int_{\S^d} \eta(w,b) w_1^{u_1} \cdots w_d^{u_d} b^{v}  dw db
$$
are fixed, for $u_1+\dots+u_d+v = \alpha$, and since the family of polynomials $(w,b) \mapsto w_1^{u_1} \cdots w_d^{u_d} b^v$ is linearly independent in $L_2(\S^d \times [-1,1])$, the same reasoning above leads to an RKHS norm which is equivalent 
the $L_2$-norm on the space of polynomials of degree less than $\alpha$ (or equal).

\section{Computing leverage scores}
\label{app:ls}

In this section, we explicitly compute leverage scores for $d=1$ and $\alpha=0$ for the two expansions.

\subsection{General solution}

We consider $d=1$, $R=1$, and the classical integral operator for the uniform distribution on $[-1,1]$:
$$
\Sigma f(x) = \frac{1}{2}\int_{-1}^1 k_1^{(0)}(x,y) f(y) dy
= \frac{1}{4} \int_{-1} ^1 f(y)dy - \frac{1}{8} \int_{-1}^1 |x-y| f(y) dy.
$$

Given a function $g \in L_2([-1,1])$, we aim to compute the leverage score:
$$
\frac{1}{2} \int_{-1}^1 g(x) \big[  ( \Sigma + \lambda \idm)^{-1} g \big](x) dx.
$$
We thus compute $f =  ( \Sigma + \lambda \idm)^{-1} g$, which is such that:
$$
g(x) = \frac{1}{4} \int_{-1} ^1 f(y)dy - \frac{1}{8} \int_{-1}^1 |x-y| f(x) dy + \lambda f(x).
$$
By taking two derivatives and using the fact that the second-order derivative of $x \mapsto |x-y|$ is $2\delta_y$, we get:
$$
g''(x) = \lambda f''(x) - \frac{1}{4} f(x).
$$
Once we know a solution $f_0$ for the ordinary differential equation above, then all solutions are obtained as 
\BEQ
\label{eq:F}
f(x) = f_0(x) + A \cosh \frac {x}{2 \sqrt{\lambda}}+B \sinh \frac {x}{2 \sqrt{\lambda}},
\EEQ
for some $A,B \in \rb$.
\paragraph{Obtaining a solution in ``closed-form''.}
We can   solve the ODE in $f$ using standard techniques~\cite{bender1999advanced}, by writing $f(x) = e^{\frac{x}{2 \sqrt{\lambda}}} a(x)$, so that
\BEAS
g''(x) & = &  \lambda  e^{\frac{x}{2 \sqrt{\lambda}}} a''(x) + {\sqrt{\lambda}} e^{\frac{x}{2 \sqrt{\lambda}}} a'(x) 
\\
e^{- \frac{x}{2 \sqrt{\lambda}}} g''(x)  & = & \lambda a''(x) +  {\sqrt{\lambda}}a'(x).
\EEAS
We then write $a'(x) = e^{- \frac{x}{ \sqrt{\lambda}}} c(x)$, so that
$$
e^{- \frac{x}{2 \sqrt{\lambda}}} g''(x)  =  \lambda e^{- \frac{x}{ \sqrt{\lambda}}} c'(x),
$$
and thus
$$
c'(x) = \frac{1}{\lambda} e^{ \frac{x}{2 \sqrt{\lambda}}} g''(x),
$$
leading to a particular solution by integration:
$$
c(x) =    \frac{1}{2\lambda} \int_{-1}^1  e^{ \frac{y}{2 \sqrt{\lambda}}} g''(y) \sign(x-y) dy.
$$
Moreover, we get a particular solution:
\BEAS
a(x) & = &   \frac{1}{2} \int_{-1}^1   e^{- \frac{y}{ \sqrt{\lambda}}} c(y) \sign(x-y) dy\\
& = &   \frac{1}{4\lambda}  \int_{-1}^1   e^{- \frac{y}{ \sqrt{\lambda}}} 
 \int_{-1}^1  e^{ \frac{t}{2 \sqrt{\lambda}}} g''(t) \sign(y-t) dt
 \sign(x-y) dy\\
 f(x)   & = &    \frac{1}{4\lambda} \int_{-1}^1  
 \int_{-1}^1  e^{\frac{x}{2 \sqrt{\lambda}}}  e^{- \frac{y}{ \sqrt{\lambda}}}  e^{ \frac{t}{2 \sqrt{\lambda}}} g''(t) \sign(y-t)  \sign(x-y) dy dt,
\EEAS
with all solutions obtained by adding $A \cosh \frac {x}{2 \sqrt{\lambda}}+B \sinh \frac {x}{2 \sqrt{\lambda}}$.

\paragraph{Finding constants $A$ and $B$.}
We have for the unique solution $f$ of $g = ( \Sigma + \lambda \idm) f$:
  \BEA
\notag  g(1) & = &  \frac{1}{4} \int_{-1} ^1 f(y)dy - \frac{1}{8} \int_{-1}^1 (1-y) f(y) dy + \lambda f(1)
  \\
 \notag   g(-1) & = &  \frac{1}{4} \int_{-1} ^1 f(y)dy - \frac{1}{8} \int_{-1}^1 (1+y) f(y) dy + \lambda f(-1) , \mbox{ leading to}\\
\label{eq:gpg}    g(1)+g(-1) & = &  \frac{1}{4} \int_{-1} ^1 f(y)dy + \lambda [ f(1)+f(-1) ] \\
\label{eq:gmg}   g(1) - g(-1) & = &  \frac{1}{4} \int_{-1} ^1 y  f(y)dy + \lambda [ f(1)-f(-1) ].
  \EEA
For $f(x) = \cosh \frac {x}{2 \sqrt{\lambda}}$, we have:
$\ds
\int_{-1}^1 f(x)dx = 4 \sqrt{\lambda} \sinh  \frac {1}{2 \sqrt{\lambda}}.
$

For $f(x) = \sinh \frac {x}{2 \sqrt{\lambda}}$, we have:
$\ds
\int_{-1}^1 x f(x)dx = 4 \sqrt{\lambda} \cosh  \frac {1}{2 \sqrt{\lambda}}
- 8 \lambda  \sinh  \frac {1}{2 \sqrt{\lambda}}.$

Thus, for our solution in \eq{F}:
\BEAS
   g(1)+g(-1) & = &  \frac{1}{4} \int_{-1} ^1 f_0(y)dy + \lambda [ f_0(1)+f_0(-1) ]  +
   A \Big[ \sqrt{\lambda} \sinh  \frac {1}{2 \sqrt{\lambda}} + 2 \lambda  \cosh \frac {1}{2 \sqrt{\lambda}} \Big]\\
    g(1) - g(-1) & = &  \frac{1}{4} \int_{-1} ^1 y  f_0(y)dy + \lambda [ f_0(1)-f_0(-1) ]
    + B  \sqrt{\lambda} \cosh  \frac {1}{2 \sqrt{\lambda}}.
  \EEAS
  Therefore, to obtain $A$ and $B$, we simply need to compute $f_0(1)$, $f_0(-1)$ as well as
  $\ds \int_{-1} ^1 f_0(y)dy $ and $\ds \int_{-1} ^1 y f_0(y)dy $.

\subsection{Neural networks}
We consider $g(x) = 1_{x>b} = (x-b)_+^0$, we then consider
$f_0(x) = 1_{x>b}\frac{1}{\lambda} \cosh \frac{x-b}{2\sqrt{\lambda}}$.
We have:
\BEAS
g(x) - \lambda f_0(x) &= &  1_{x>b} \big[ 1 - \cosh \frac{x-b}{2\sqrt{\lambda}}  \big] \\
g'(x) - \lambda f_0'(x) & = & - \frac{1}{2\sqrt{\lambda}} 1_{x>b} \sinh \frac{x-b}{2\sqrt{\lambda}} 
\\
g''(x) - \lambda f_0''(x) & = &-  \frac{1}{4 \lambda} 1_{x>b} \cosh \frac{x-b}{2\sqrt{\lambda}} = -\frac{1}{4} f_0(x),
\EEAS
and thus $f_0$ is a particular solution. We have:
\BEAS
f_0(1) + f_0(-1) & = &  f_0(1)-f_0(-1)  = f_0 (1)  = \frac{1}{\lambda}\cosh \frac{1-b}{2\sqrt{\lambda}} \\
%& = & 
%\frac{1}{\lambda} \big[ \cosh \frac{1}{2\sqrt{\lambda}}
% \cosh \frac{b}{2\sqrt{\lambda}} - \sinh \frac{1}{2\sqrt{\lambda}}
% \sinh \frac{b}{2\sqrt{\lambda}} \big] \\
 \int_{-1}^1 f_0(x) dx
  & = & \int_{b}^1 
 \frac{1}{\lambda} \cosh \frac{x-b}{2\sqrt{\lambda}} dx 
 = \frac{ 2}{\sqrt{\lambda}} \sinh \frac{1-b}{2\sqrt{\lambda}}  \\
% & = & \frac{2}{\sqrt{\lambda}}
% \big[ \sinh \frac{1}{2\sqrt{\lambda}}
% \cosh \frac{b}{2\sqrt{\lambda}} - \cosh \frac{1}{2\sqrt{\lambda}}
% \sinh \frac{b}{2\sqrt{\lambda}} \big]  \\
 \int_{-1}^1x  f_0(x) dx & = & \int_{b}^1 
 \frac{1}{\lambda} x \cosh \frac{x-b}{2\sqrt{\lambda}} dx 
 =\int_{b}^1 
 \frac{1}{\lambda} b \cosh \frac{x-b}{2\sqrt{\lambda}} dx 
 + \int_{b}^1 
 \frac{1}{\lambda} (x -b) \cosh \frac{x-b}{2\sqrt{\lambda}} dx  \\
 & = & \frac{ 2 b}{\sqrt{\lambda}} \sinh \frac{1-b}{2\sqrt{\lambda}}  
 + \frac{ 2 }{\sqrt{\lambda}} 
 \Big[
 (1-b) \sinh \frac{1-b}{2\sqrt{\lambda}}   - 2 \sqrt{\lambda} \cosh \frac{1-b}{2\sqrt{\lambda}}  
 \Big]
\\
& = &  \frac{ 2 }{\sqrt{\lambda}} 
 \Big[
 \sinh \frac{1-b}{2\sqrt{\lambda}}   - 2 \sqrt{\lambda} \cosh \frac{1-b}{2\sqrt{\lambda}}  
 \Big].
%\\
%& = &  \frac{ 2 }{\sqrt{\lambda}} 
% \Big[
%\sinh \frac{1}{2\sqrt{\lambda}}
% \cosh \frac{b}{2\sqrt{\lambda}} - \cosh \frac{1}{2\sqrt{\lambda}}
% \sinh \frac{b}{2\sqrt{\lambda}}   - 2 \sqrt{\lambda}
% \big( \cosh \frac{1}{2\sqrt{\lambda}}
% \cosh \frac{b}{2\sqrt{\lambda}} - \sinh \frac{1}{2\sqrt{\lambda}}
% \sinh \frac{b}{2\sqrt{\lambda}} \big) 
% \Big]
%\\
%& = &  \frac{ 2 }{\sqrt{\lambda}} 
% \Big[
%\sinh \frac{1}{2\sqrt{\lambda}}
% \cosh \frac{b}{2\sqrt{\lambda}} - \cosh \frac{1}{2\sqrt{\lambda}}
% \sinh \frac{b}{2\sqrt{\lambda}}   - 2 \sqrt{\lambda}
% \big( \cosh \frac{1}{2\sqrt{\lambda}}
% \cosh \frac{b}{2\sqrt{\lambda}} - \sinh \frac{1}{2\sqrt{\lambda}}
% \sinh \frac{b}{2\sqrt{\lambda}} \big) 
% \Big]
\EEAS
Thus
\BEAS
   \frac{1}{2} & = &  \frac{1}{2 \sqrt{\lambda} } \sinh \frac{1-b}{2\sqrt{\lambda}}  + \cosh \frac{1-b}{2\sqrt{\lambda}}   +
   2\lambda A \Big[ \frac{1}{2\sqrt{\lambda}} \sinh  \frac {1}{2 \sqrt{\lambda}} +    \cosh \frac {1}{2 \sqrt{\lambda}} \Big]\\
     \frac{1}{2} & = & 
      \frac{ 1 }{2 \sqrt{\lambda}} 
 \sinh \frac{1-b}{2\sqrt{\lambda}}   
    + B  \sqrt{\lambda} \cosh  \frac {1}{2 \sqrt{\lambda}},
    % \\
%     \frac{1}{2} & = & 
%      \frac{ 1 }{2 \sqrt{\lambda}} \Big[
%      \sinh \frac{1}{2\sqrt{\lambda}}
% \cosh \frac{b}{2\sqrt{\lambda}} - \cosh \frac{1}{2\sqrt{\lambda}}
% \sinh \frac{b}{2\sqrt{\lambda}}    \Big]
%    + B  \sqrt{\lambda} \cosh  \frac {1}{2 \sqrt{\lambda}}.
  \EEAS
  which allows to solve for $A$ and $B$.
Moreover
\BEAS
\int_{-1}^1 f(x) g(x) dx
& = & \int_{-1}^1 g(x) \big[ f_0(x) + A \cosh \frac {x}{2 \sqrt{\lambda}}+B \sinh \frac {x}{2 \sqrt{\lambda}} \big]dx\\
& = & 
 \frac{ 2 }{  \sqrt{\lambda}} 
 \sinh \frac{1-b}{2\sqrt{\lambda}}    
 +  2A \sqrt{\lambda} \big[ \sinh \frac {1}{2 \sqrt{\lambda}}
 -\sinh \frac {b}{2 \sqrt{\lambda}}
 \big]  
 +  2B \sqrt{\lambda} \big[ \cosh \frac {1}{2 \sqrt{\lambda}}
 -\cosh \frac {b}{2 \sqrt{\lambda}}
 \big]
\\
& = & 
 \frac{ 2 }{  \sqrt{\lambda}} 
 \sinh \frac{1-b}{2\sqrt{\lambda}}    
 +  \frac{1}{\sqrt{\lambda}}
 \frac{ 
 \big[  \frac{1}{2} -  \frac{1}{2 \sqrt{\lambda} } \sinh \frac{1-b}{2\sqrt{\lambda}}  - \cosh \frac{1-b}{2\sqrt{\lambda}}   \big]}{ \frac{1}{2\sqrt{\lambda}} \sinh  \frac {1}{2 \sqrt{\lambda}} +    \cosh \frac {1}{2 \sqrt{\lambda}} }  \big[ \sinh \frac {1}{2 \sqrt{\lambda}}
 -\sinh \frac {b}{2 \sqrt{\lambda}}
 \big]   \\
 & & 
 +   \frac{1 -  \frac{ 1 }{ \sqrt{\lambda}} 
 \sinh \frac{1-b}{2\sqrt{\lambda}}   }{\cosh \frac {1}{2 \sqrt{\lambda}}} \big[ \cosh \frac {1}{2 \sqrt{\lambda}}
 -\cosh \frac {b}{2 \sqrt{\lambda}}
 \big],
\EEAS
which is our desired quantity (multiplied by 2). This quantity is maximized at $b=0$, for which we have the value:
\BEAS
 &&\frac{ 2 }{  \sqrt{\lambda}} 
 \sinh \frac{1}{2\sqrt{\lambda}}    
 +  \frac{1}{\sqrt{\lambda}}
 \frac{ 
 \big[  \frac{1}{2} -  \frac{1}{2 \sqrt{\lambda} } \sinh \frac{1}{2\sqrt{\lambda}}  - \cosh \frac{1}{2\sqrt{\lambda}}   \big]}{ \frac{1}{2\sqrt{\lambda}} \sinh  \frac {1}{2 \sqrt{\lambda}} +    \cosh \frac {1}{2 \sqrt{\lambda}} }  \big[ \sinh \frac {1}{2 \sqrt{\lambda}}
  \big]   
 +   \frac{1 -  \frac{ 1 }{ \sqrt{\lambda}} 
 \sinh \frac{1}{2\sqrt{\lambda}}   }{\cosh \frac {1}{2 \sqrt{\lambda}}} \big[ \cosh \frac {1}{2 \sqrt{\lambda}}
 -1
 \big] \\
 & = & 
    \frac{1}{\sqrt{\lambda}}
 \frac{ 
 \big[  \frac{1}{2} +  \frac{1}{2 \sqrt{\lambda} } \sinh \frac{1}{2\sqrt{\lambda}}  - \cosh \frac{1}{2\sqrt{\lambda}}   \big]}{ \frac{1}{2\sqrt{\lambda}} \sinh  \frac {1}{2 \sqrt{\lambda}} +    \cosh \frac {1}{2 \sqrt{\lambda}} }  \big[ \sinh \frac {1}{2 \sqrt{\lambda}}
  \big]   
 +   \big[  {1 -  \frac{ 1 }{ \sqrt{\lambda}} 
 \sinh \frac{1}{2\sqrt{\lambda}}   } \big] \cdot \big[ 1 - \frac{1}{\cosh \frac {1}{2 \sqrt{\lambda}}}
 \big] \\
& = & 
\Big[ 1 + \frac{1/2}{ \frac{1}{2\sqrt{\lambda}} \sinh  \frac {1}{2 \sqrt{\lambda}} +    \cosh \frac {1}{2 \sqrt{\lambda}} }  \Big]
\frac{1}{\sqrt{\lambda}}   \sinh \frac {1}{2 \sqrt{\lambda}}
 +   \big[  {1 -  \frac{ 1 }{ \sqrt{\lambda}} 
 \sinh \frac{1}{2\sqrt{\lambda}}   } \big] \cdot \big[ 1 - \frac{1}{\cosh \frac {1}{2 \sqrt{\lambda}}}
 \big] \\
& = & 
 \frac{\frac{1}{2\sqrt{\lambda}}   \sinh \frac {1}{2 \sqrt{\lambda}}
}{ \frac{1}{2\sqrt{\lambda}} \sinh  \frac {1}{2 \sqrt{\lambda}} +    \cosh \frac {1}{2 \sqrt{\lambda}} }   
 +     \big[ 1 - \frac{1}{\cosh \frac {1}{2 \sqrt{\lambda}}}
 \big] 
 +     \frac{ 1 }{ \sqrt{\lambda}} 
  \frac{\sinh \frac{1}{2\sqrt{\lambda}} }{\cosh \frac {1}{2 \sqrt{\lambda}}}.
\EEAS
The maximal leverage score has thus order $ \frac{1}{2\sqrt{\lambda}}$.

\subsection{Fourier feature}
We consider $g(x) = e^{i\omega x}$ so that we can obtain both $\cos \omega x$ and $\sin \omega x$. Then we can take
$f_0(x) = \frac{\omega^2}{\lambda \omega^2 + \frac{1}{4}} e^{i\omega x}$ as a special solution,
since
$$
\lambda f_0''(x) - \frac{1}{4} f_0(x) = \omega^2 \frac{-  \lambda \omega^2 - \frac{1}{4} }  {\lambda \omega^2 + \frac{1}{4}}e^{i\omega x} = g''(x).
$$
We get, from \eq{gpg} and \eq{gmg}:
\BEAS
   2 \cos \omega & = &  \frac{\omega^2}{4\lambda \omega^2+1}
   \frac{1}{i\omega} 2 i \sin \omega 
    + 
   \frac{\lambda \omega^2}{\lambda \omega^2 + \frac{1}{4}} 2 \cos \omega +
   A \Big[ \sqrt{\lambda} \sinh  \frac {1}{2 \sqrt{\lambda}} + 2 \lambda  \cosh \frac {1}{2 \sqrt{\lambda}} \Big]\\
    2i \sin \omega & = &  \frac{\omega^2}{4\lambda \omega^2+1}\Big[ \frac{1}{\omega^2} e^{i\omega x} ( 1 - i \omega x)  \Big]_{-1} ^1   +   \frac{\lambda \omega^2}{\lambda \omega^2 + \frac{1}{4}} 2 i \sin \omega
    + B  \sqrt{\lambda} \cosh  \frac {1}{2 \sqrt{\lambda}}.
  \EEAS
This leads to explicit formulas for the constants $A$ and $B$:
\BEAS
    \frac{2\cos \omega - 2 \omega \sin\omega  }{4\lambda \omega^2+1}  & = &      
   A \Big[ \sqrt{\lambda} \sinh  \frac {1}{2 \sqrt{\lambda}} + 2 \lambda  \cosh \frac {1}{2 \sqrt{\lambda}} \Big]\\
       \frac{2 i \sin \omega }{4\lambda \omega^2+1}  & = &  \frac{1}{4\lambda \omega^2+1}\Big[ 
       2i \sin \omega    - 2i \omega \cos \omega  \Big] 
    + B  \sqrt{\lambda} \cosh  \frac {1}{2 \sqrt{\lambda}} , \mbox{ leading to} \\
         \frac{2 i \omega \cos \omega }{4\lambda \omega^2+1}  & = &   B  \sqrt{\lambda} \cosh  \frac {1}{2 \sqrt{\lambda}}.  
  \EEAS
  We then get
  \BEAS
  A& = &  \frac{2\cos \omega - 2 \omega \sin\omega  }{4\lambda \omega^2+1} \frac{1}{ \sqrt{\lambda} \sinh  \frac {1}{2 \sqrt{\lambda}} + 2 \lambda  \cosh \frac {1}{2 \sqrt{\lambda}}}\\
  \frac{B}{i} & = &  \frac{2  \omega \cos \omega }{4\lambda \omega^2+1} \frac{1}{ \sqrt{\lambda} \cosh  \frac {1}{2 \sqrt{\lambda}}} .
  \EEAS
Thus, the solution for $g(x)= \cos \omega x$ is
$ \ds f(x) =  \frac{\omega^2}{\lambda \omega^2 + \frac{1}{4}}  \cos \omega x + A \cosh \frac{x}{2\sqrt{\lambda}}$,
while the solution for $g(x) = \sin \omega x$
is
$ \ds f(x) =  \frac{\omega^2}{\lambda \omega^2 + \frac{1}{4}}  \sin \omega x + \frac{B}{i} \sinh \frac{x}{2\sqrt{\lambda}}$.

Thus, we can compute for $g(x)= \cos \omega x$
\BEAS
\int_{-1}^1 f(x) g(x) dx
& = & \int_{-1}^1 \cos \omega x \Big[ \frac{\omega^2}{\lambda \omega^2 + \frac{1}{4}}  \cos \omega x + A \cosh \frac{x}{2\sqrt{\lambda}}\Big]dx\\
& = & \frac{\omega^2}{\lambda \omega^2 + \frac{1}{4}}
\Big( 1 +  \frac{1}{2} \frac{\sin \omega}{\omega}   \Big)
+ A
\int_{-1}^1 \cos \omega x    \cosh \frac{x}{2\sqrt{\lambda}} dx\\
& = & \frac{\omega^2}{\lambda \omega^2 + \frac{1}{4}}
\Big( 1 +  \frac{1}{2} \frac{\sin \omega}{\omega}   \Big)
+ \frac{2 A}{\omega^2 + \frac{1}{4\lambda}}
\Big[ 
\frac{1}{2\sqrt{\lambda}}\cos \omega  \sinh 
\frac{1}{2\sqrt{\lambda}} + \omega \sin \omega  \cosh 
\frac{1}{2\sqrt{\lambda}}
\Big].
\EEAS
For $g(x) = \sin \omega x$, we get:
\BEAS
\int_{-1}^1 f(x) g(x) dx
& = & \int_{-1}^1 \sin \omega x \Big[ \frac{\omega^2}{\lambda \omega^2 + \frac{1}{4}}  \sin \omega x + \frac{B}{i} \sinh \frac{x}{2\sqrt{\lambda}}\Big]dx\\
& = & \frac{\omega^2}{\lambda \omega^2 + \frac{1}{4}}
\Big(1  - \frac{1}{2} \frac{\sin \omega}{\omega}   \Big)
+ \frac{B}{i}
\int_{-1}^1 \sin \omega x    \sinh \frac{x}{2\sqrt{\lambda}} dx\\
& = & \frac{\omega^2}{\lambda \omega^2 + \frac{1}{4}}
\Big(  1 -  \frac{1}{2} \frac{\sin \omega}{\omega}   \Big)
+ \frac{2B/i}{\omega^2 + \frac{1}{4\lambda}}
\Big[ 
\frac{1}{2\sqrt{\lambda}}\sin \omega  \cosh 
\frac{1}{2\sqrt{\lambda}} - \omega \cos \omega  \sinh 
\frac{1}{2\sqrt{\lambda}}
\Big].
\EEAS
We thus obtain the two leverage scores (divided by 2). We notice that the two leverage scores tend to $1/(2\lambda)$ for $\omega$ tending to infinity, which is the largest value for all $\omega$.
 
\subsection{Empirical comparisons} 
As detailed in \cite[Appendix A]{pauwels2018relating}, we can estimate the leverage scores from a grid in $[-1,1]$ with $n$ points by computing $\sum_{i,j=1}^n \varphi(x_i,v) \varphi(x_j,v) \big[( K + n \lambda \idm)^{-1} \big]_{ij}$, and compare with the theoretical expression found above, which match. See \myfig{levscore}.

\begin{figure}[h]
\begin{center}
\includegraphics[scale=.4]{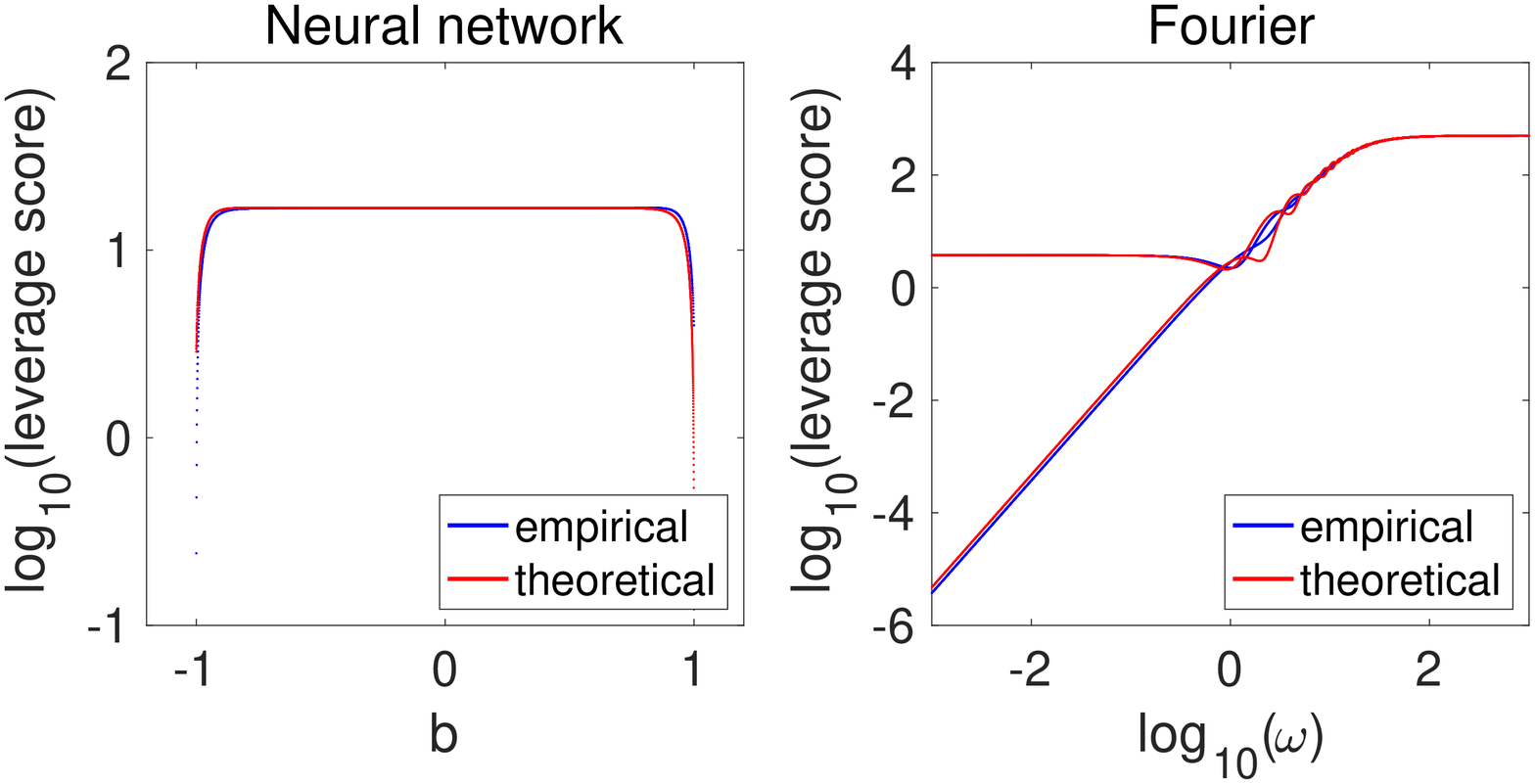}
\end{center}

\vspace*{-.5cm}

\caption{Comparison of empirical and theoretical leverage scores for neural network feature (left) and Fourier features (right). We used $\lambda = 10^{-3}$ and $n = 4096$.
\label{fig:levscore}}
\end{figure}

    \bibliography{splines}

\begin{thebibliography}{10}

\bibitem{hastie2009elements}
Trevor Hastie, Robert Tibshirani, and Jerome~H. Friedman.
\newblock {\em The Elements of Statistical Learning: Data Mining, Inference,
  and Prediction}.
\newblock Springer, 2009.

\bibitem{wasserman2006all}
Larry Wasserman.
\newblock {\em All of Nonparametric Statistics}.
\newblock Springer Science \& Business Media, 2006.

\bibitem{gyorfi2002distribution}
L{\'a}szl{\'o} Gy{\"o}rfi, Michael Kohler, Adam Krzyzak, and Harro Walk.
\newblock {\em A Distribution-free Theory of Nonparametric Regression}.
\newblock Springer, 2002.

\bibitem{neal1995bayesian}
Radford~M. Neal.
\newblock {\em Bayesian Learning for Neural Networks}.
\newblock PhD thesis, University of Toronto, 1995.

\bibitem{rahimi2008random}
Ali Rahimi and Benjamin Recht.
\newblock Random features for large-scale kernel machines.
\newblock In {\em Advances in Neural Information Processing Systems}, pages
  1177--1184, 2008.

\bibitem{scholkopf2002learning}
Bernhard Sch{\"o}lkopf and Alexander~J. Smola.
\newblock {\em Learning with Kernels: Support Vector Machines, Regularization,
  Optimization, and Beyond}.
\newblock MIT Press, 2002.

\bibitem{rudi2017generalization}
Alessandro Rudi and Lorenzo Rosasco.
\newblock Generalization properties of learning with random features.
\newblock In {\em Advances in Neural Information Processing Systems}, pages
  3215--3225, 2017.

\bibitem{le2007continuous}
Nicolas Le~Roux and Yoshua Bengio.
\newblock Continuous neural networks.
\newblock In {\em Artificial Intelligence and Statistics}, pages 404--411,
  2007.

\bibitem{duchon1977splines}
Jean Duchon.
\newblock Splines minimizing rotation-invariant semi-norms in {S}obolev spaces.
\newblock In {\em Constructive Theory of Functions of Several Variables}, pages
  85--100. Springer, 1977.

\bibitem{buhmann2003radial}
Martin~D. Buhmann.
\newblock {\em Radial Basis Functions: Theory and Implementations}, volume~12.
\newblock Cambridge University Press, 2003.

\bibitem{nair2010rectified}
Vinod Nair and Geoffrey~E. Hinton.
\newblock Rectified linear units improve restricted {B}oltzmann machines.
\newblock In {\em International Conference on Machine Learning}, 2010.

\bibitem{bach2017equivalence}
Francis Bach.
\newblock On the equivalence between kernel quadrature rules and random feature
  expansions.
\newblock {\em Journal of Machine Learning Research}, 18(1):714--751, 2017.

\bibitem{cho2009kernel}
Youngmin Cho and Lawrence~K. Saul.
\newblock Kernel methods for deep learning.
\newblock In {\em Advances in Neural Information Processing Systems}, 2009.

\bibitem{bach2017breaking}
Francis Bach.
\newblock Breaking the curse of dimensionality with convex neural networks.
\newblock {\em Journal of Machine Learning Research}, 18(1):629--681, 2017.

\bibitem{pmlr-v130-scetbon21b}
Meyer Scetbon and Zaid Harchaoui.
\newblock A spectral analysis of dot-product kernels.
\newblock In {\em International Conference on Artificial Intelligence and
  Statistics}, pages 3394--3402, 2021.

\bibitem{daniely2016toward}
Amit Daniely, Roy Frostig, and Yoram Singer.
\newblock Toward deeper understanding of neural networks: The power of
  initialization and a dual view on expressivity.
\newblock {\em Advances in Neural Information Processing Systems}, 29, 2016.

\bibitem{kristiadi2021infinite}
Agustinus Kristiadi, Matthias Hein, and Philipp Hennig.
\newblock An infinite-feature extension for {B}ayesian {ReLU} nets that fixes
  their asymptotic overconfidence.
\newblock {\em Advances in Neural Information Processing Systems}, 34, 2021.

\bibitem{berlinet2011reproducing}
Alain Berlinet and Christine Thomas-Agnan.
\newblock {\em Reproducing Kernel Hilbert Spaces in Probability and
  Statistics}.
\newblock Springer Science \& Business Media, 2011.

\bibitem{adams2003sobolev}
Robert~A. Adams and John J.~F. Fournier.
\newblock {\em Sobolev Spaces}.
\newblock Academic Press, 2003.

\bibitem{wendland2004scattered}
Holger Wendland.
\newblock {\em Scattered Data Approximation}.
\newblock Cambridge University Press, 2004.

\bibitem{wahba1990spline}
Grace Wahba.
\newblock {\em Spline Models for Observational Data}.
\newblock SIAM, 1990.

\bibitem{friedlander1998introduction}
Friedrich~Gerard Friedlander, Mark~Suresh Joshi, M.~Joshi, and Mohan~C. Joshi.
\newblock {\em Introduction to the {T}heory of {D}istributions}.
\newblock Cambridge University Press, 1998.

\bibitem{williams2006gaussian}
Christopher K.~I. Williams and Carl~Edward Rasmussen.
\newblock {\em Gaussian Processes for Machine Learning}.
\newblock MIT Press, 2006.

\bibitem{bender1999advanced}
Carl~M. Bender, Steven Orszag, and Steven~A. Orszag.
\newblock {\em Advanced Mathematical Methods for Scientists and Engineers I:
  Asymptotic Methods and Perturbation Theory}, volume~1.
\newblock Springer Science \& Business Media, 1999.

\bibitem{pauwels2018relating}
Edouard Pauwels, Francis Bach, and Jean-Philippe Vert.
\newblock Relating leverage scores and density using regularized {C}hristoffel
  functions.
\newblock {\em Advances in Neural Information Processing Systems}, 31, 2018.

\end{thebibliography}
 
  \end{document}